\documentclass[10pt,twocolumn,letterpaper]{article}

\usepackage[pagenumbers]{cvpr} %

\usepackage{graphicx}
\usepackage{amsmath}
\usepackage{amssymb}
\usepackage{booktabs}
\usepackage{amsthm}

\usepackage{mathtools}
\usepackage{empheq}

\usepackage{xcolor}

\usepackage[ruled]{algorithm2e}
\usepackage[pagebackref,breaklinks,colorlinks]{hyperref}

\newtheorem{theorem}{Theorem}[section]
\newtheorem{lemma}[theorem]{Lemma}
\newtheorem{assumption}{Assumption}
\newcommand{\norm}[1]{\left\lVert#1\right\rVert}

\usepackage[pagebackref,breaklinks,colorlinks]{hyperref}

\usepackage[capitalize]{cleveref}
\crefname{section}{Sec.}{Secs.}
\Crefname{section}{Section}{Sections}
\Crefname{table}{Table}{Tables}
\crefname{table}{Tab.}{Tabs.}

\def\cvprPaperID{} %
\def\confName{CVPR}
\def\confYear{2022}

\begin{document}

\title{Surface-Aligned Neural Radiance Fields for Controllable 3D Human Synthesis}

\author{
Tianhan Xu\thanks{Work done while the first author interned at Preferred Networks, Inc.}
\\
The University of Tokyo\\
{\tt\small tianhan.xu@mi.t.u-tokyo.ac.jp}
\and
Yasuhiro Fujita \qquad Eiichi Matsumoto\\
Preferred Networks, Inc.\\
{\tt\small \{fujita, matsumoto\}@preferred.jp}
}
\maketitle

\begin{abstract}

\vspace{-0.5em}

We propose a new method for reconstructing controllable implicit 3D human models from sparse multi-view RGB videos. Our method defines the neural scene representation on the mesh surface points and signed distances from the surface of a human body mesh. We identify an indistinguishability issue that arises when a point in 3D space is mapped to its nearest surface point on a mesh for learning surface-aligned neural scene representation. To address this issue, we propose projecting a point onto a mesh surface using a barycentric interpolation with modified vertex normals. Experiments with the ZJU-MoCap and Human3.6M datasets show that our approach achieves a higher quality in a novel-view and novel-pose synthesis than existing methods. We also demonstrate that our method easily supports the control of body shape and clothes. Project page: \url{https://pfnet-research.github.io/surface-aligned-nerf/}.
\end{abstract}
\vspace{-0.8em}

\section{Introduction}
\label{sec:intro}

Human body modeling is a long-studied topic for its wide range of real-world applications. In visual applications such as movies or games, which often require free-viewpoint rendering, it is common to expect 3D human models to have controllable properties such as pose, shape, and clothes. Because manually designing high-quality 3D human models is usually labor-intensive, increasing studies~\cite{ma2020cape, alldieck2018video, Alldieck_2019_CVPR, alldieck2019tex2shape, 2021narf, liu2021neural, peng2021neural, peng2021animatable} have proposed the reconstruction of 3D human models using only 2D observations. In this paper, we focus on the free-viewpoint 3D human synthesis with the above controllable properties from sparse multi-view RGB videos. 

Early approaches~\cite{Carranza2003actor, Vlasic2008articulated} deformed the pre-scanned template meshes with a skeleton for modeling a human shape and/or texture. Parametric 3D human models~\cite{scape2005, SMPL:2015, SMPL-X:2019} have been proposed to reconstruct rough human meshes with pose and body shape estimation. Subsequent studies~\cite{ma2020cape, alldieck2018video, Alldieck_2019_CVPR, alldieck2019tex2shape} introduced more features, such as per-vertex deformation or texture, to express richer details. Such parametric model-based approaches, despite having good controllable properties, show limitations in representing clothed humans, particularly when the actual shape differs significantly from the base parametric mesh estimation.

Neural radiance fields (NeRF)~\cite{mildenhall2020nerf}, a new form of 3D scene representation, has recently become the new baseline method in 3D reconstruction for its photorealistic rendering results of novel camera view. NeRF represents the scene as a continuous volumetric representation using a neural network to regress the color and density at a given query point from a given view direction. Several approaches~\cite{2021narf, peng2021neural, peng2021animatable, liu2021neural} have been proposed to incorporate knowledge from a statistical 3D human model and its pose estimation with NeRF. They differ in how a query point is transformed and represented. Deformation-based approaches~\cite{peng2021animatable, liu2021neural} use a deformation field to transform the query point from the observation space to a pose-independent canonical space and then build NeRF in the canonical space. Other approaches transform the query points into local coordinate systems~\cite{2021narf} or to a latent code representation~\cite{peng2021neural}, with the help of human pose estimation.

In this paper, we propose a new approach for achieving a dynamic human reconstruction by combining parametric 3D body models such as SMPL~\cite{SMPL:2015} with NeRF. The basic idea of our approach is straightforward and simple: we propose building a NeRF on the mesh surface.
We devise an algorithm to map a query point to a mesh surface point with a signed height, which can represent the local position of the point with respect to the mesh.
Using the information of the surface point position and the signed height of a query point as input, we build a surface-aligned NeRF that is aligned with the mesh surface; thus, it can be easily deformed or controlled according to the base SMPL model.

Our approach has the following advantages: First, with the help of the devised mapping algorithm, our method does not rely on a learned deformation field, saving the number of learned parameters. Second, the models reconstructed using our method can be controlled directly by the SMPL parameters, that is, both poses and body shapes. Third, due to the surface-aligned property, our approach shows a better generalization ability for a novel human pose synthesis. 

In summary, our contributions are as follows:
\begin{itemize}
\item We propose an algorithm that can injectively map a spatial point to a novel surface-aligned representation that consists of a projected surface point and a signed height to the mesh surface.
\item We propose novel surface-aligned neural radiance fields using the proposed mapping, which can be easily controlled using the SMPL parameters. Compared to existing methods, our approach shows a better generalization performance on a novel view and novel pose synthesis while supporting manipulations such as changes to the body shape and clothes.
\end{itemize}

\section{Related Work}
\label{sec:relatedwork}

\paragraph{Human body modeling.}
Early studies proposed parametric mesh models such as SCAPE~\cite{scape2005} or SMPL~\cite{SMPL:2015} to model the shape of the human body. SMPL recovers the human mesh given the skeleton pose and body shape and is commonly used as the basis for controllable human modeling. However, one of the problems with the parametric model is that it can only model the naked human body. Subsequent studies proposed a per-vertex deformation based on the SMPL for better modeling of clothed-human details~\cite{ma2020cape, alldieck2018video, Alldieck_2019_CVPR, alldieck2019tex2shape}.
Due to the strong expressiveness of the implicit 3D representation, recent studies proposed combining it with the SMPL to better capture the shape and appearance of a human body~\cite{huang_arch_2020, bhatnagar2020ipnet}. In~\cite{Yang_2021_CVPR}, the shape, pose, and skinning weights are represented with neural implicit functions to model a dynamic human body. In~\cite{deng2020nasa}, a pose-conditioned implicit occupancy function is used to predict the shape of an articulated human body. There are also studies~\cite{pifuSHNMKL19, saito2020pifuhd, natsume_siclope:_2019} that addressed reconstruction of clothed humans from a single image.

\paragraph{Neural radiance fields for a dynamic human body.}
Recently, neural radiance fields (NeRF)~\cite{mildenhall2020nerf} has become a common building block for photorealistic novel view synthesis. Some studies have improved on the original NeRF to enable it to model dynamic scenes~\cite{park2021nerfies, pumarola2020d, li2020scene, du2021nerflow, tretschk2021nonrigid, park2021hypernerf}. The key idea of these methods is to introduce a deformation field that maps the observation space to a canonical space and builds a NeRF in the canonical space. However, optimizing both the deformation field and NeRF is an under-constrained problem that can cause implausible results or artifacts~\cite{park2021nerfies}.

For modeling a dynamic human body, recent studies~\cite{peng2021neural, peng2021animatable, 2021narf, liu2021neural} have proposed the use of prior knowledge of human pose and skinning weights of SMPL~\cite{SMPL:2015} to ease the learning of a deformation field. 
Animatable NeRF~\cite{peng2021animatable} uses the skinning weight of SMPL~\cite{SMPL:2015} to predict the neural blend shape fields. Neural Actor~\cite{liu2021neural} predicts the texture map from the posed mesh and utilizes it as additional information to help predict the deformation fields.
However, learning a deformation field leads to a significant increase in network parameters, as well as the cost and difficulty of training and inference. Neural Body~\cite{peng2021neural} uses the structured latent code anchored at the SMPL vertices to encode the pose information. However, this approach shows poor performance for a novel pose synthesis because the encoding process of the latent code strongly relies on the training pose. NARF~\cite{2021narf} uses the human skeleton to transform spatial points into bone coordinates to learn a locally-defined NeRF, which has some conceptual similarity to our approach. 
However, because such bone coordinates have no surface information, they do not fully utilize a human body prior from SMPL nor support manipulations such as body shape control.

\section{Method}
\label{sec:method}
\begin{figure*}[t]
  \centering
  \includegraphics[width=\linewidth]{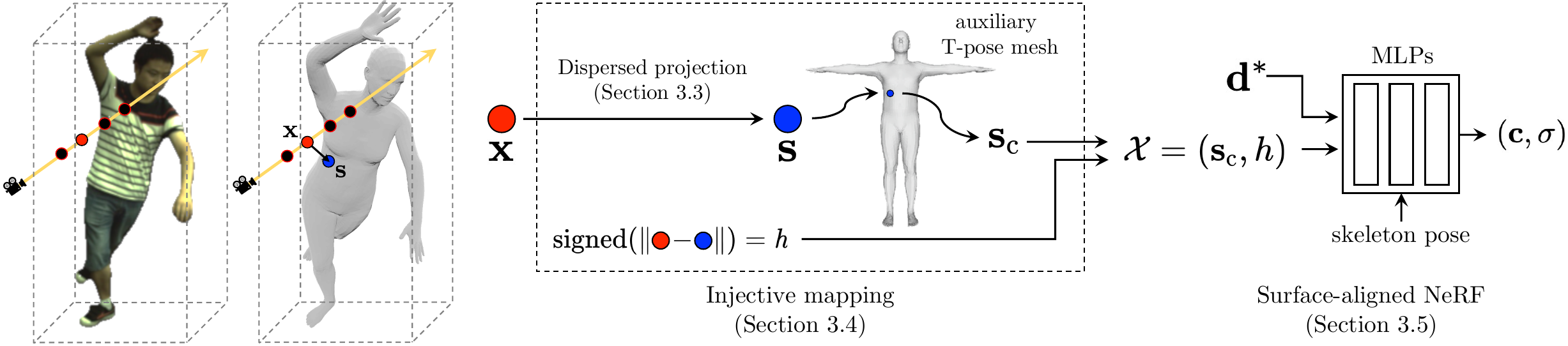}
  \caption{\textbf{An overview of our approach.} Given a query point $\mathbf{x} \in \mathbb{R}^3$, we use the proposed dispersed projection to project it onto a point $\mathbf{s} \in \mathbb{R}^3$ on the mesh surface to obtain a surface-aligned representation $\mathcal{X}$. The representation $\mathcal{X}$ and the view direction $\mathbf{d}^*$ are then input into the NeRF to compute the color $\mathbf{c}$ and density $\sigma$ of the query point $\mathbf{x}$. }
  \label{fig:pipeline}
\end{figure*}

We aim to reconstruct a 3D human model from sparse multi-view videos that can be photorealistically rendered with a fully controllable camera view, human pose, and body shape. We assume that, for each video frame, we have approximate human pose and shape information (\eg, SMPL~\cite{SMPL:2015} parameters) and the foreground mask using off-the-shelf methods~\cite{gong2018instance}.

An overview of the proposed approach is shown in \cref{fig:pipeline}. We propose feeding a NeRF (\cref{sec:nerf_rev}) with a novel representation (\cref{sec:local}) consisting of a surface point and signed height, which is calculated through the proposed dispersed projection (\cref{sec:dispersed}) given a query point. We call this variant of thet NeRF a surface-aligned NeRF (\cref{sec:local_nerf}).

\subsection{Neural radiance fields revisited}
\label{sec:nerf_rev}

Neural Radiance Fields (NeRF)~\cite{mildenhall2020nerf} uses a neural network to model a scene in a continuous volumetric representation. A neural network $f$ calculates the RGB color $\mathbf{c}$ and the density $\sigma$ at a given spatial position $\mathbf{x}$ from a given viewpoint $\mathbf{d}$, which can be written as: 
\begin{equation}
    f: (\mathbf{x}, \mathbf{d}) \rightarrow (\mathbf{c}, \sigma)
\end{equation}
NeRF-based methods typically use position information $\mathbf{x} \in \mathbb{R}^3$ in a fixed three-dimensional Euclidean space (\eg, the world coordinate). To model dynamic human bodies, recent studies have proposed some transformation or new representation of position information as the input of the NeRF. Specifically, NARF~\cite{2021narf} transforms the position information of spatial points into each bone coordinate and uses all of them as input; Neural Body~\cite{peng2021neural} uses a neural network to compute a latent code representing the position information relative to the body mesh and uses it as input. 

\subsection{Surface-aligned representation}
\label{sec:local}

We assume that the motion of the human body can be divided into the following three components according to the order of dominance: 1) surface-aligned components that follow the deformation of the body mesh, 2) pose-dependent components (\eg, clothes deformation caused by pose change), 3) the rest of the time-varying components (\eg, fluttering from the wind). 
We propose a new representation, called \textit{surface-aligned representation}, for representing the positions of the query point $\textbf{x} \in \mathbb{R}^3$ with respect to the body mesh, thus describing the most dominant mesh-following deformation components.
Specifically, our proposed surface-aligned representation consists of two components: 1) The position information of surface point $\textbf{s}$ obtained by projecting $\textbf{x}$ onto the mesh surface.
2) The signed distance $h$ between $\textbf{x}$ and the projection point $\textbf{s}$, which represents the ``height'' of $\textbf{x}$ to the mesh surface.
To represent $\textbf{s}$ in a pose-independent manner, we map it to the surface point on the shared T-pose body mesh.
Specifically, for $\textbf{s}$ on the posed mesh surface, we compute the point with the same barycentric coordinate of the same triangle face on the T-pose mesh surface and use its spatial coordinates $\textbf{s}_c = (x_c, y_c, z_c)$. 
The proposed surface-aligned representation can be written as:
\begin{equation}
    \mathcal{X} = (\textbf{s}_c, h) = (x_c, y_c, z_c, h) \in \mathbb{R}^4
    \label{eq:local_repr}
\end{equation}
We map the spatial point $\mathbf{x}$ to the surface-aligned representation $\mathcal{X}$ and use it as input to the NeRF.

\begin{figure}[t]
  \centering
  \includegraphics[width=\linewidth]{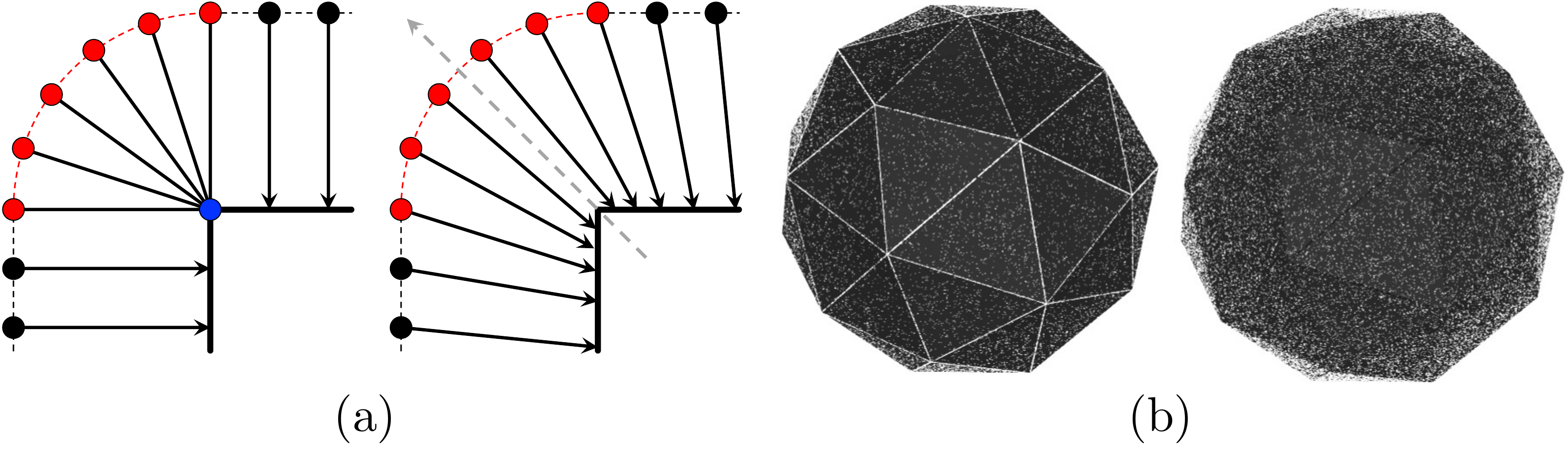}
  \caption{\textbf{Comparison of different projection methods.} 
  \textbf{(a)} Toy example in 2D. We project the points of the same height onto the surface (thick line). Left: nearest point projection projects all the red points onto the vertex, thus obtaining the same surface-aligned representation. Right: proposed dispersed projection (here equals to the barycentric interpolated projection), in which all points can be projected onto different and distinguishable points.
  \textbf{(b)} We randomly sample points in space and project them onto the mesh surface. Left: nearest point projection, projections tend to concentrate on the vertices or edges. Right: proposed dispersed projection, projections are well distributed.}
  \label{fig:projection_compare}
  \vspace{-0.5em}
\end{figure}

\begin{figure*}[t]
  \centering
  \includegraphics[width=\linewidth]{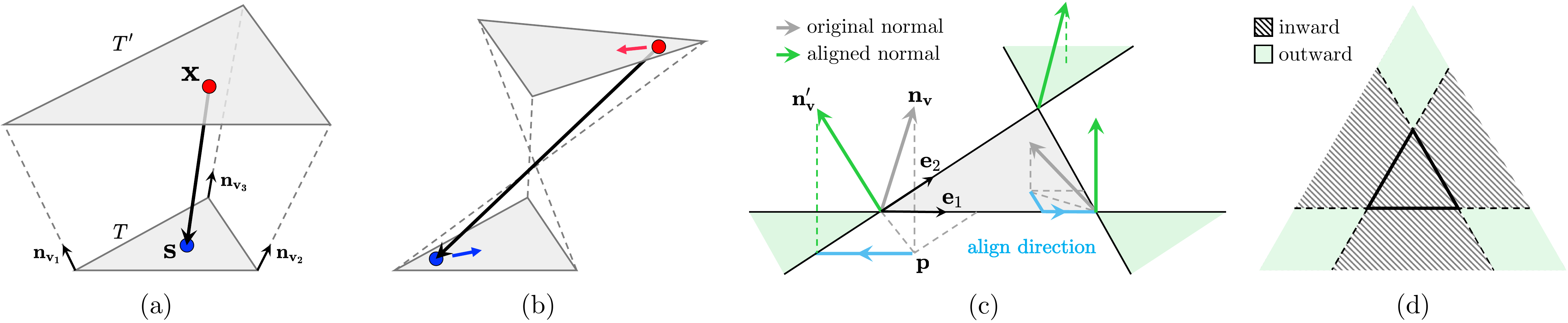}
  \caption{\textbf{Illustration of barycentric interpolated projection and vertex normal alignment.} (a) Illustration of barycentric interpolated projection (b) Example of ``inverted'' projection that occurs when some vertex normals are facing inward. (c) Illustration of vertex normal alignment. (d) Inward and outward regions, viewed from directly above the triangle. }
  \label{fig:projection_detail}
\end{figure*}

\subsection{Dispersed projection}
\label{sec:dispersed}
We propose a method of  projecting a spatial point $\textbf{x}$ onto the mesh surface to obtain $\textbf{s}$, called \textit{dispersed projection}, which consists of two key components: \textit{barycentric interpolated projection} and \textit{vertex normal alignment}. We first introduce them separately and then show the detailed steps of the dispersed projection. Note that in the following \textit{we only consider the case of $h > 0$ (a point outside the mesh)} because the idea is exactly the same as when points are inside (in that case, we can revert the normal direction).
We assume that each vertex normals are from the inside of the mesh to the outside, that is, formally, the inner products are positive for the vertex normal with the face normals of all triangle faces that share this vertex.

\vspace{-1em}
\paragraph{Nearest point projection.}
An obvious way to project a spatial point onto a surface is what we call \textit{nearest point projection}, that is, calculating the point on the surface that is closest to a given point.
Animatable NeRF~\cite{peng2021animatable} and Neural Actor~\cite{liu2021neural} rely on this way of projection for learning a deformation field and/or utilizing a texture map.
While using nearest point projection to obtain $\mathbf{s}$ is straightforward, we note that the mapping $\mathbf{x} \rightarrow \mathcal{X}$ using nearest point projection is not an injection (or one-to-one mapping), thus causing an indistinguishability issue in optimizing NeRF. As the 2D illustration on the left side of \cref{fig:projection_compare}(a) shows, the red points on the arc (with the same height $h$) are all projected onto the blue point, being indistinguishable in the resulting representation. In the case of a 3D mesh, as shown in \cref{fig:projection_compare}(b)~Left, projected points are concentrated on vertices or edges.
The key issue is that NeRF would output the same color and density for points with different spatial positions but the same surface-aligned representation $\mathcal{X}$, failing to capture sufficient details.

\vspace{-1em}
\paragraph{Barycentric interpolated projection.} 

We propose a novel projection method to address the indistinguishability issue in nearest point projection, named \textit{barycentric interpolated projection}.
We illustrate the method in \cref{fig:projection_detail}(a). For a spatial point $\mathbf{x}$ and a triangle $T$ on the mesh surface, we take the plane that passes through $\mathbf{x}$ and is parallel to $T$. The intersection of this plane with the three vertex normals forms a parallel triangle $T'$. Then, we compute the barycentric coordinate of $\mathbf{x}$ in the triangle $T'$. Finally, using this barycentric coordinate for the triangle $T$, we can obtain the point $\mathbf{s}$ on the mesh surface through a barycentric interpolation.

\vspace{-1em}
\paragraph{Vertex normal alignment.} 
While barycentric interpolated projection works as expected when the vertex normals all appear to be facing \textit{outward}, that is, intuitively, the parallel triangle $T'$ is larger than and therefore able to ``cover'' $T$, allowing nearby points to be projected onto $T$ (see \cref{fig:projection_detail}(a)).
However, when not all the vertex normals face \textit{outward}, as shown in \cref{fig:projection_detail}(b), the resulting projection may be ``inverted''.
To address this, we propose to perform a procedure named \textit{vertex normal alignment}. First, we provide a formal definition of \textit{inward} and \textit{outward}. Let us consider the orthogonal projections of vertex normals of a triangle on a plane. As shown in \cref{fig:projection_detail}(c) and (d), we divide the plane into \textit{inward} and \textit{outward} regions. If a vertex normal's projection falls in the \textit{inward} and \textit{outward} regions, we call the vertex normal \textit{inward} and \textit{outward}, respectively.
We align all the \textit{inward} normals of a triangle as follows.
First, we consider the orthogonal projection point $\mathbf{p}$ of $\mathbf{n_v}$. We move $\mathbf{n_v}$ along the directions $\mathbf{e}_1$ and  $\mathbf{e}_2$ of the two edges until $\mathbf{p}$ reaches the \textit{outward} region, obtaining the aligned direction $\mathbf{n'_v}$. Finally, we normalize its length to $1$ as the aligned vertex normal. 
Note that vertex normal alignment is performed separately for each triangle, and thus even the same vertex, shared by different triangles, may have differently aligned vertex normals.
After vertex normal alignment, the above-mentioned issues of the barycentric interpolated projection can be avoided, and we can obtain more reasonable and distinguishable projection points.

\vspace{-1em}
\paragraph{Detailed steps of dispersed projection.} We show the detailed steps of projecting a spatial point $\mathbf{x}$ to the mesh surface point $\mathbf{s}$ by combining barycentric interpolated projection and vertex normal alignment.

\begin{enumerate}
\itemsep 0.4em
    \item Compute the nearest point $\Tilde{\mathbf{s}}$ on the mesh surface.
    \item Find a set $\mathcal{T}$ of all triangles containing $\Tilde{\mathbf{s}}$. When $\Tilde{\mathbf{s}}$ falls inside a triangle, obviously only that triangle contains $\Tilde{\mathbf{s}}$; however, when it falls on a vertex or an edge, there will be multiple triangles containing $\Tilde{\mathbf{s}}$.
    \item Apply vertex normal alignment to all the triangles in $\mathcal{T}$.
    \item For each $T \in \mathcal{T}$, if $\mathbf{x}$ is not inside its parallel triangle $T'$, remove it from $\mathcal{T}$. There should be at least one triangle in $\mathcal{T}$ such that $\mathbf{x}$ is inside its $T'$.
    \item Apply barycentric interpolated projection for each $T \in \mathcal{T}$ to obtain the projected surface point $\{ \mathbf{s} \}$.
    \item Choose the nearest point to $\mathbf{x}$ from $\{ \mathbf{s} \}$ as the final projection point.
\end{enumerate}

\subsection{Injective mapping to a surface-aligned representation}
By projecting $\mathbf{x}$ onto the surface point $\mathbf{s}$ using the proposed dispersed projection, as described in \cref{sec:local}, we can compute the surface-aligned representation $\mathcal{X}$ of $\mathbf{x}$. We emphasize that, using the proposed dispersed projection, the mapping $\mathbf{x} \rightarrow \mathcal{X}$ is an injection under certain conditions. Thus, different spatial points are mapped to different representations, eliminating the indistinguishability issue. The proof can be found in the supplementary materials.

\subsection{Surface-aligned neural radiance fields}
\label{sec:local_nerf}
Our goal is to build neural radiance fields aligned with the mesh surface that can change with the mesh deformation.
To this end, we feed the surface-aligned representation $\mathcal{X}$ (\cref{eq:local_repr}) of a spatial point $\mathbf{x}$ to NeRF as input. 
Furthermore, we consider the view direction in both the world and local coordinates. Here, the local coordinates refer to the coordinates formed by the normal, tangent, and bitangent directions on the surface of the posed mesh at $\mathbf{s}$. 
We denote the view direction in the world coordinates and the local coordinates as $\mathbf{d}$ and $\mathbf{d}_l$, respectively.
We concatenate them together $\mathbf{d}^* = \mathrm{concat}(\mathbf{d}, \mathbf{d}_l)$ and input to the NeRF.
This can be written as:
\begin{equation}
    \label{eq:naive_nerf}
    F_\Theta: (\mathcal{X}, \mathbf{d}^*) \rightarrow (\mathbf{c}, \sigma)
\end{equation}
where $\Theta$ denotes the model parameters. 
However, surface-aligned representation $\mathcal{X}$ can only capture the motion that strictly follows the body mesh deformation; \cref{eq:naive_nerf} cannot model some subtle pose-dependent deformation such as loose clothing.
To address this issue, we also condition our model on the skeleton pose parameter $p$ to enable learning the pose-dependent deformation. 
Specifically, we use an encoder network to encode the pose information as a latent code $\mathbf{z}_p$ and feed it as an additional input to the NeRF:
\begin{equation}
\label{eq:pose_nerf}
    F_\Theta: (\mathcal{X}, \mathbf{d}^*, \mathbf{z}_p) \rightarrow (\mathbf{c}, \sigma)
\end{equation}
This is our proposed surface-aligned NeRF.

\subsection{Training}
\label{sec:train}
We use volume rendering~\cite{kajiya1986render, mildenhall2020nerf} to synthesize images with a NeRF.
Following the previous studies~\cite{mildenhall2020nerf, peng2021neural, peng2021animatable}, we minimize the per-pixel mean squared error (MSE) between the rendered image and the ground truth image.
We use SMPL~\cite{SMPL:2015} as our model's underlying body mesh model with estimated shape and pose parameters.
To reduce the impact of an inaccurate SMPL parameter estimation, we optimize the pose parameter of the SMPL simultaneously with the NeRF using Adam optimizer~\cite{kingma:adam}.
Additional training details can be found in the supplementary material.

\section{Experiments}
\label{sec:experiments}

\begin{figure*}[t]
  \centering
  \includegraphics[width=\linewidth]{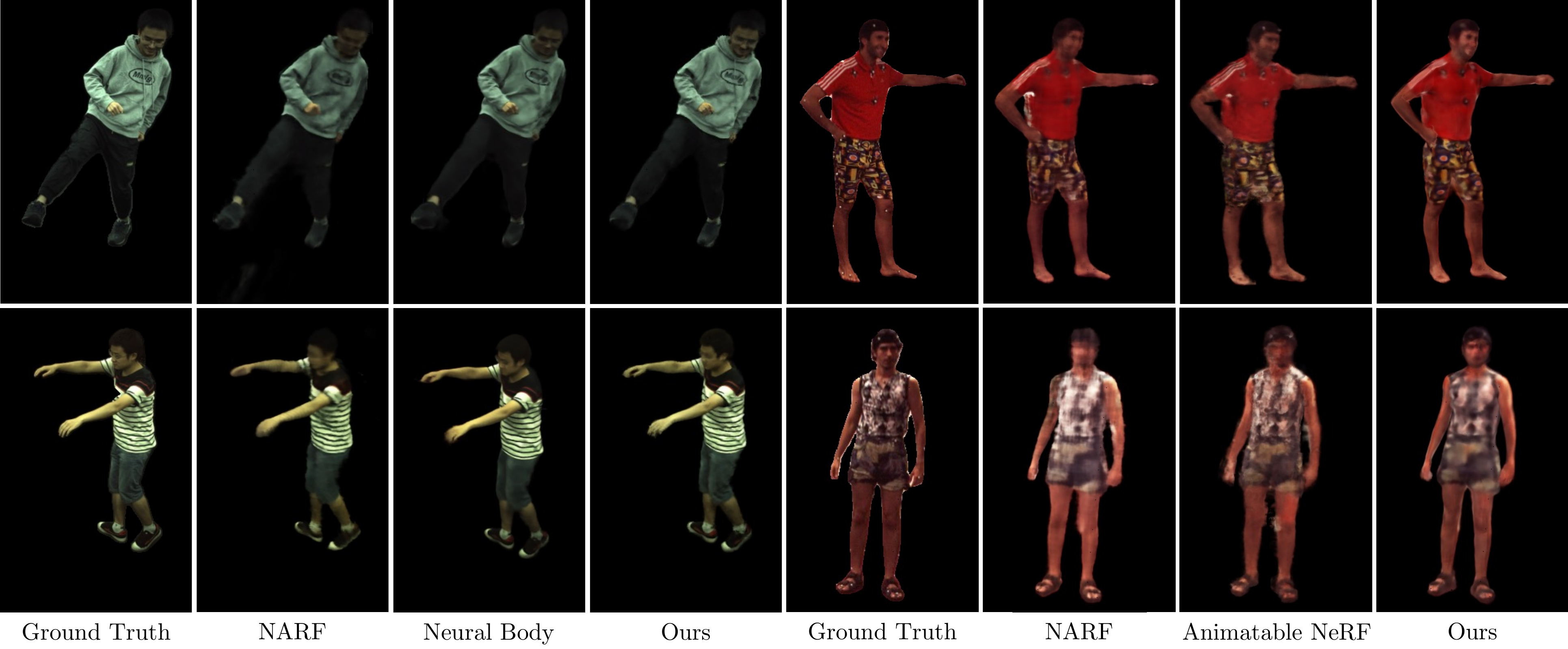}
  \caption{\textbf{Qualitative results of novel view synthesis for the training pose.} Left: ZJU-MoCap dataset. Right: Human3.6M dataset.}
  \label{fig:novel_view}
\end{figure*}

\begin{figure*}[t]
  \centering
  \includegraphics[width=\linewidth]{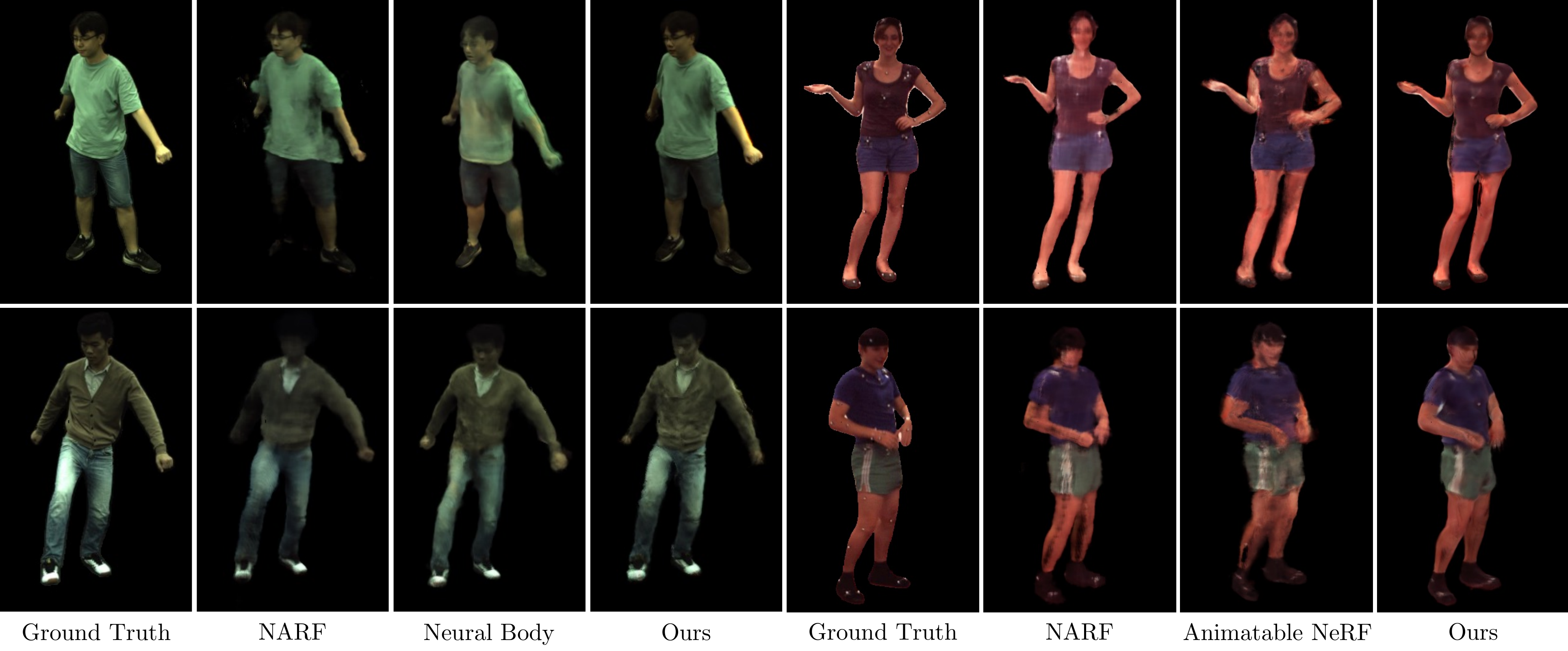}
  \caption{\textbf{Qualitative results of novel view synthesis for the unseen pose.} Left: ZJU-MoCap dataset. Right: Human3.6M dataset.}
  \label{fig:novel_pose}
\end{figure*}

\begin{table*}
\begin{center}
\resizebox{\textwidth}{!}{
\begin{tabular}{c|ccc|ccc|ccc|ccc}
& \multicolumn{6}{c|}{\textbf{Training pose}} & \multicolumn{6}{c}{\textbf{Unseen pose}} \\ \hline
& \multicolumn{3}{c|}{PSNR} & \multicolumn{3}{c|}{SSIM} & \multicolumn{3}{c|}{PSNR} & \multicolumn{3}{c}{SSIM} \\
\hline
& NARF\cite{2021narf} & NB\cite{peng2021neural} & Ours & NARF\cite{2021narf} & NB\cite{peng2021neural} & Ours & NARF\cite{2021narf} & NB\cite{peng2021neural} & Ours & NARF\cite{2021narf} & NB\cite{peng2021neural} & Ours \\
\hline
Twirl  & 29.38 & 30.56 & \textbf{31.32} & 0.967 & 0.971 & \textbf{0.974} & 22.20 & 23.95 & \textbf{24.33} & 0.872 & 0.905 & \textbf{0.908} \\
Taichi & 24.22 & 27.24 & \textbf{27.25} & 0.930 & 0.962 & 0.962 & 19.70 & 19.56 & \textbf{19.87} & 0.859 & 0.852 & \textbf{0.863} \\
Swing1 & 27.53 & \textbf{29.44} & 29.29 & 0.929 & 0.946 & 0.946 & 25.43 & 25.76 & \textbf{26.27} & 0.912 & 0.909 & \textbf{0.927} \\
Swing2 & 27.54 & 28.44 & \textbf{28.76} & 0.928 & 0.940 & \textbf{0.941} & 24.03 & 23.80 & \textbf{24.96} & 0.884 & 0.878 & \textbf{0.900} \\
Swing3 & 26.56 & \textbf{27.58} & 27.50 & 0.925 & \textbf{0.939} & 0.938 & 23.84 & 23.25 & \textbf{24.24} & 0.901 & 0.893 & \textbf{0.908} \\
Warmup & 25.89 & 27.64 & \textbf{27.67} & 0.931 & 0.951 & \textbf{0.954} & 24.14 & 23.91 & \textbf{25.34} & 0.906 & 0.909 & \textbf{0.928} \\
Punch1 & 25.98 & 28.60 & \textbf{28.81} & 0.895 & 0.931 & 0.931 & 25.24 & 25.68 & \textbf{27.30} & 0.877 & 0.881 & \textbf{0.905} \\
Punch2 & 24.78 & 25.79 & \textbf{26.08} & 0.915 & 0.928 & \textbf{0.929} & 22.58 & 21.60 & \textbf{23.08} & 0.885 & 0.870 & \textbf{0.890} \\
Kick   & 26.42 & 27.59 & \textbf{27.77} & 0.913 & 0.926 & \textbf{0.927} & 23.53 & 23.90 & \textbf{24.43} & 0.872 & 0.870 & \textbf{0.889} \\
\hline                                                
average. & 26.48 & 28.10 & \textbf{28.27} & 0.926 & 0.944 & \textbf{0.945} & 23.41 & 23.49 & \textbf{24.42} & 0.885 & 0.885 & \textbf{0.902} \\
\end{tabular}
}
\end{center}
\vspace{-1em}
\caption{\textbf{Results of the ZJU-MoCap dataset ~\cite{peng2021neural} in terms of PSNR and SSIM.} Higher is better. ``NB" means Neural Body.}
\label{table:results_zju}
\end{table*}

\begin{table*}
\begin{center}
\resizebox{\textwidth}{!}{
\begin{tabular}{c|ccc|ccc|ccc|ccc}
& \multicolumn{6}{c|}{\textbf{Training pose}} & \multicolumn{6}{c}{\textbf{Unseen pose}} \\ \hline
& \multicolumn{3}{c|}{PSNR} & \multicolumn{3}{c|}{SSIM} & \multicolumn{3}{c|}{PSNR} & \multicolumn{3}{c}{SSIM} \\
\hline
& NARF\cite{2021narf} & AN\cite{peng2021animatable} & Ours & NARF\cite{2021narf} & AN\cite{peng2021animatable} & Ours & NARF\cite{2021narf} & AN\cite{peng2021animatable} & Ours & NARF\cite{2021narf} & AN\cite{peng2021animatable} & Ours \\
\hline
S1  & 21.41 & 22.05 & \textbf{23.71} & 0.891 & 0.888 & \textbf{0.915} & 20.19 & 21.37 & \textbf{22.67} & 0.864 & 0.868 & \textbf{0.890} \\
S5 & \textbf{25.24} & 23.27 & 24.78 & \textbf{0.914} & 0.892 & 0.909 & \textbf{23.91} & 22.29 & 23.27 & \textbf{0.891} & 0.875 & 0.881 \\
S6 & 21.47 & 21.13 & \textbf{23.22} & 0.871 & 0.854 & \textbf{0.881} & 22.47 & 22.59 & \textbf{23.23} & 0.883 & 0.884 & \textbf{0.888} \\
S7 & 21.36 & 22.50 & \textbf{22.59} & 0.899 & 0.890 & \textbf{0.905} & 20.66 & 22.22 & \textbf{22.51} & 0.876 & 0.878 & \textbf{0.898} \\
S8 & 22.03 & 22.75 & \textbf{24.55} & 0.904 & 0.898 & \textbf{0.922} & 21.09 & 21.78 & \textbf{23.06} & 0.887 & 0.882 & \textbf{0.904} \\
S9 & 25.11 & 24.72 & \textbf{25.31} & 0.906 & 0.908 & \textbf{0.913} & 23.61 & 23.72 & \textbf{23.84} & 0.881 & 0.886 & \textbf{0.889} \\
S11 & 24.35 & 24.55 & \textbf{25.83} & 0.902 & 0.902 & \textbf{0.917} & 23.95 & 23.91 & \textbf{24.19} & 0.885 & 0.889 & \textbf{0.891} \\
\hline                                                
average. & 23.00 & 23.00 & \textbf{24.28} & 0.898 & 0.890 & \textbf{0.909} & 22.27 & 22.55 & \textbf{23.25} & 0.881 & 0.880 & \textbf{0.892} \\
\end{tabular}
}
\end{center}
\vspace{-1em}
\caption{\textbf{Results of the Human3.6M dataset~\cite{h36m} in terms of PSNR and SSIM.} Higher is better. ``AN" means Animatable NeRF.}
\label{table:results_h36m}
\end{table*}

\begin{figure}[t]
  \centering
  \includegraphics[width=\linewidth]{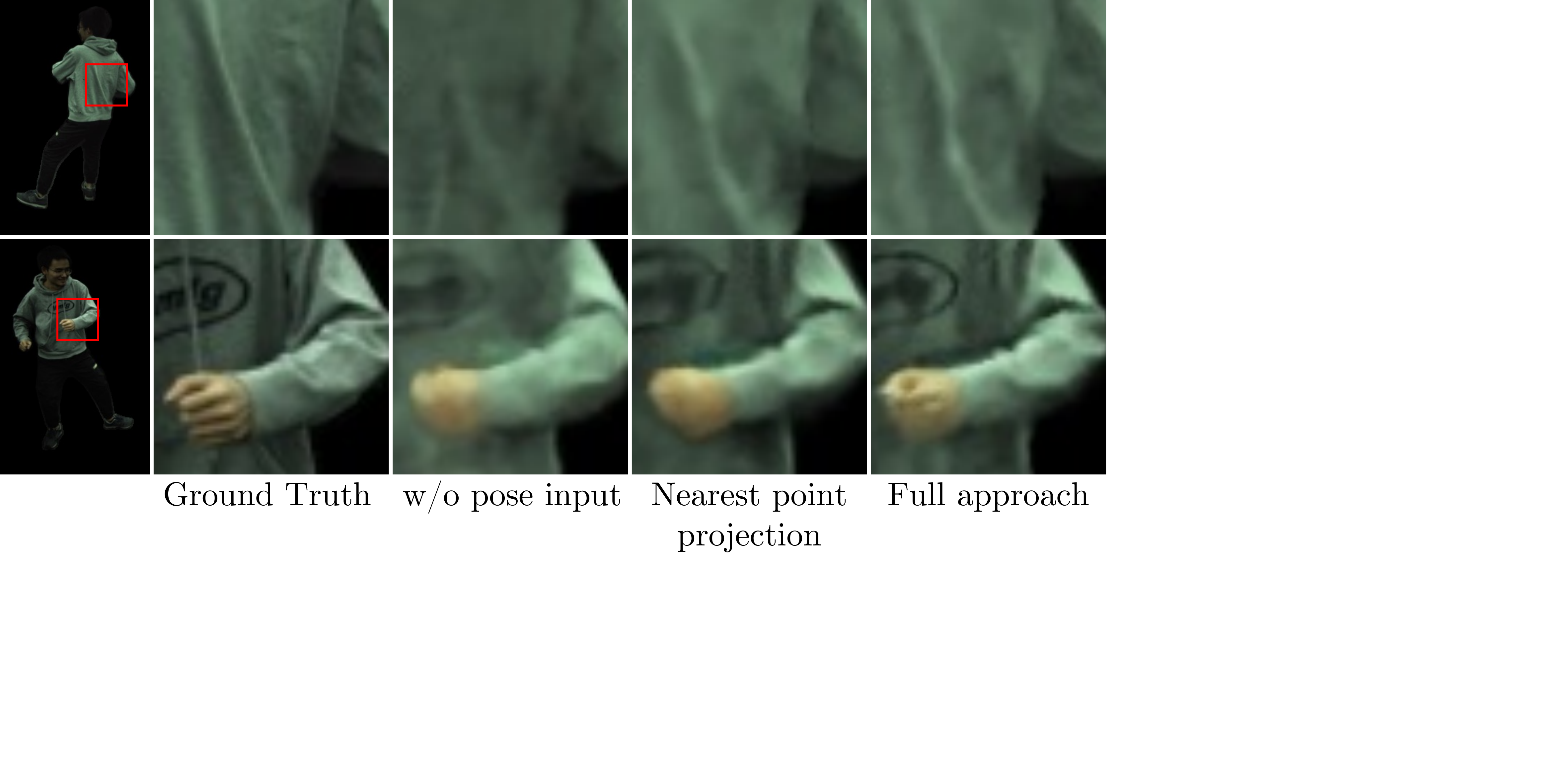}
  \vspace*{-1.4em}
  \caption{\textbf{Qualitative results of ablation studies using sequence ``Kick'' of {ZJU-MoCap} dataset.} ``w/o pose input'' indicates that skeleton pose is not used as input to NeRF. ``Nearest point projection'' means we use nearest point projection instead of proposed dispersed projection to calculate the surface-aligned representation.}
  \label{fig:ablation}
\end{figure}

\begin{figure}[t]
  \centering
  \includegraphics[width=\linewidth]{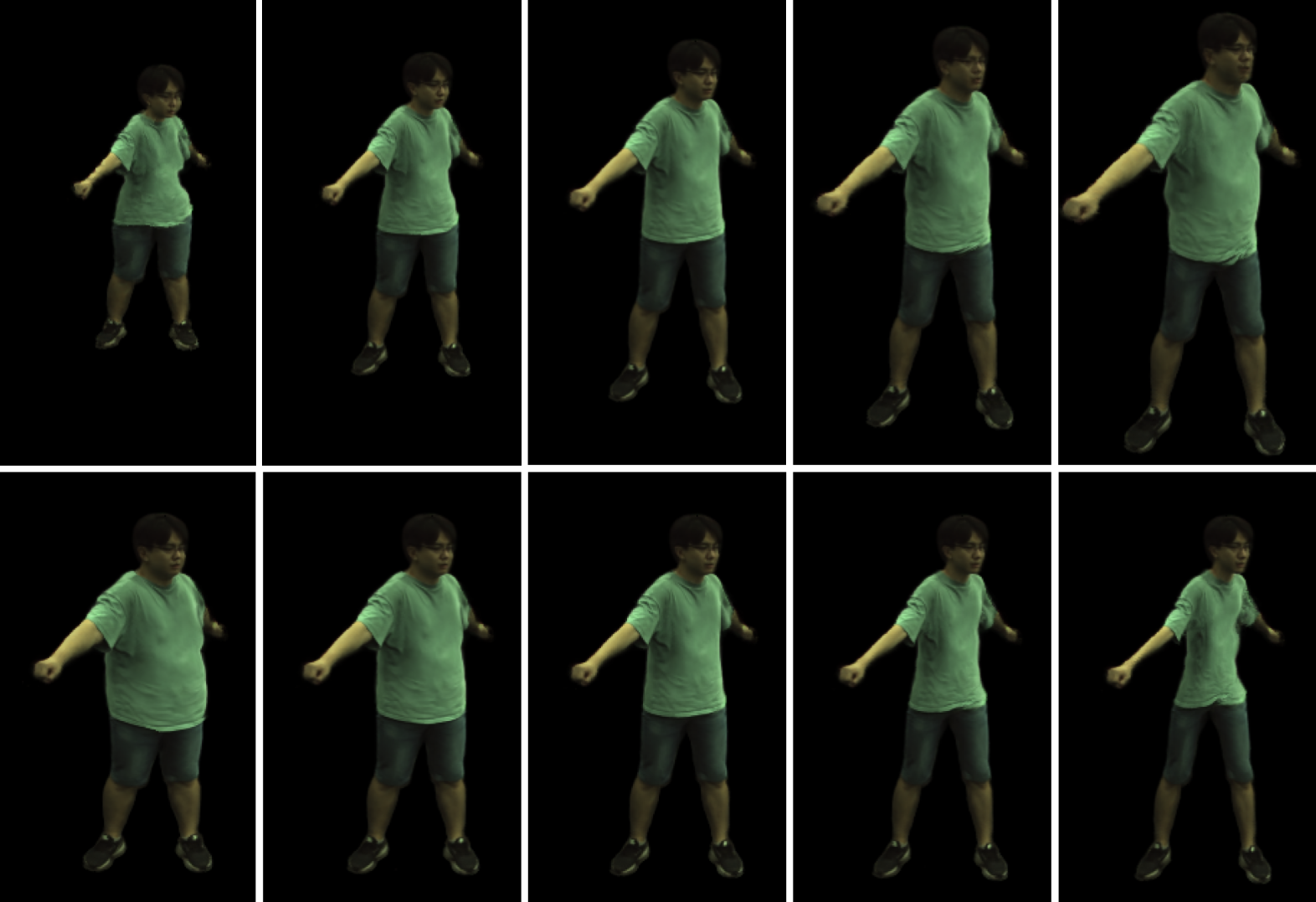}
  \caption{\textbf{Qualitative results of body shape control}. Top: change PC1 from $-4$ to $+4$. Bottom: change PC2 from $-4$ to $+4$.}
  \label{fig:shape_cont}
  \vspace*{-0.3em}
\end{figure}

\begin{figure*}[t]
  \centering
  \includegraphics[width=\linewidth]{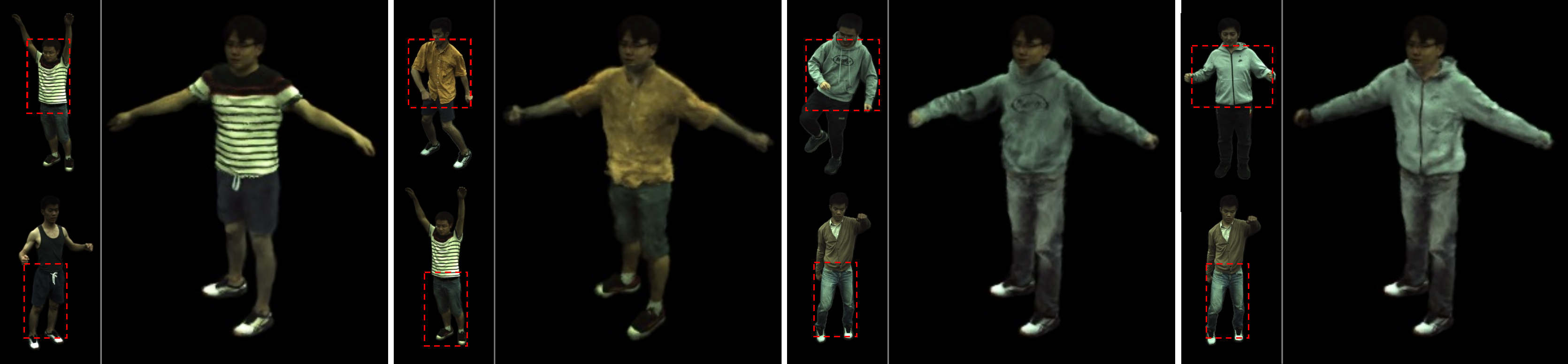}
  \caption{\textbf{Qualitative results of changing clothes}. Since the surface-aligned property of our method, we can use different trained NeRFs for different mesh surface areas. The small images on the left indicate the source appearance and the large image on the right indicates the synthesized image with novel appearance combination.}
  \label{fig:clothes_change}
  \vspace*{-0.7em}
\end{figure*}

\subsection{Datasets}
\paragraph{ZJU-MoCap~\cite{peng2021neural}} records human motion using 21 synchronized cameras and uses markerless motion capture to obtain human poses. We use four roughly evenly spaced cameras for training and the remaining cameras for testing novel view synthesis. We divide the motion sequence frames into ``training pose'' and ``unseen pose'', and use the former for training the model and the latter for testing the performance under unseen human pose.
We follow~\cite{peng2021neural} for data preprocessing and more training and testing details.

\paragraph{Human3.6M~\cite{h36m}} uses four synchronized cameras to record complex human movements and marker-based motion capture to obtain human poses. 
We use three of the cameras as training view and the remaining one for testing. Similarly, we divide the dataset into the training/unseen pose sequences. For detailed settings, we refer to~\cite{peng2021animatable}.

\paragraph{Evaluation metrics.} Following~\cite{mildenhall2020nerf}, we use two standard metrics to evaluate the performance of image synthesis: peak signal-to-noise ratio (PSNR) and structural similarity index measure (SSIM).

\subsection{Results on image synthesis}
We compare the performance of our approach with that of state-of-the-art methods. For the ZJU-Mocap dataset, we compare with Neural Body~\cite{peng2021neural}, which uses structured latent code to model the appearance of the human body. For the Human3.6M dataset, we compare with Animatable NeRF~\cite{peng2021animatable}, which learns the neural blend weight field and builds a NeRF within a canonical space. In addition, we compare with NARF~\cite{2021narf}, which takes a similar NeRF approach with an input representation based on the bone coordinates. The results are presented in \cref{table:results_zju} and \cref{table:results_h36m}. 

\paragraph{Training pose.} For the ZJU-MoCap dataset, our approach outperforms~\cite{2021narf} and shows a slightly better performance compared to~\cite{peng2021neural}, although these studies use per-frame optimization for better modeling of the training pose. For the Human3.6M dataset, our approach outperforms both~\cite{2021narf} and~\cite{peng2021animatable} by a large margin. The results show that our surface-aligned NeRF can more easily capture the shape and appearance of the human body for the training pose. The qualitative results are presented in \cref{fig:novel_view}.
From the synthesized images, we see that the proposed method has better facial details compared to the others.

\paragraph{Unseen pose.} For both datasets, the performance of our approach almost consistently outperforms~\cite{peng2021neural},~\cite{peng2021animatable}, and~\cite{2021narf}. 
Its superior performance may be attributed to its surface-aligned property as it enables fully utilizing the power of SMPL to express body mesh deformations for unseen poses.
The qualitative results are presented in \cref{fig:novel_pose}.
From the synthesized images, we see that the proposed method preserves rich details, such as folds of clothes, even for poses unseen during training, while other methods not only fail to restore details but also produce obvious artifacts.

\subsection{Ablation studies}
We use sequence ``Kick'' of the ZJU-MoCap dataset where the actual shape differs significantly from the SMPL mesh estimation (because of the loose hoodie) for an ablation study on the performance of novel view synthesis. The results are shown in \cref{fig:ablation} and \cref{table:ablation}.

\paragraph{Impact of skeleton pose conditioning.}
We train a model without the skeleton pose input in \cref{eq:pose_nerf}, which gives lower PSNR/SSIM results and blurry synthesized images.
As discussed in \cref{sec:local}, while we assume that the majority of body movements are mesh-following deformations, there may also be pose-dependent deformations that cannot be captured by the body mesh. 
Also, while we do not explicitly model the time-varying deformation components (such as using per-frame embedding in \cite{peng2021neural, peng2021animatable}), neural networks could implicitly model such components by inferring time from the skeleton pose information.

\vspace{-1em}
\paragraph{Impact of the dispersed projection.}
We use nearest point projection instead of the proposed dispersed projection for training. 
As shown in \cref{fig:ablation}, using nearest point projection leads to artifacts and loss of details.
As discussed in \cref{sec:dispersed}, the dispersed projection addresses the indistinguishability issue arising from nearest point projection.
From the illustration of the nearest point projection (\cref{fig:projection_compare}), it is easy to imagine that the farther a point is from the mesh surface, the more likely it is to be projected onto a vertex or an edge by nearest point projection. Based on this observation, we hypothesize that when the estimated mesh is close to the real 3D shape, the indistinguishability may not be so problematic; however, when there are points far from the mesh surface, such as thick clothes or fingers, it could significantly incur the model performance, which may explain the differences.

\begin{table}
\begin{center}
\begin{tabular}{c|cc}
& PSNR & SSIM \\
\hline
w/o pose input & 26.58 & 0.915 \\
Nearest point projection & 27.38 & 0.923 \\
\hline
Full approach & \textbf{27.77} & \textbf{0.927}
\end{tabular}
\end{center}
\vspace{-0.8em}
\caption{\textbf{Results of ablation studies using sequence ``Kick'' of {ZJU-MoCap} dataset in terms of PSNR and SSIM.} Higher is better. Also refer to \cref{fig:ablation}.}
\label{table:ablation}
\vspace{-0.8em}
\end{table}

\subsection{Body shape control}
Because the proposed NeRF is aligned to the body mesh surface, we can control the body shape of modeled humans by manipulating the body mesh surface. Specifically, the SMPL has 10 shape parameters obtained by principal component analysis (PCA), and by controlling them, we can obtain meshes of different body shapes. \cref{fig:shape_cont} shows the synthesized images after we changed the first and second principal components, PC1 and PC2.

\subsection{Changing clothes}
As we can replace a part of the mesh texture with another texture, our method can also replace NeRF according to the different surface areas to achieve a similar effect of ``changing clothes''. 

Suppose that we already have several proposed surface-aligned NeRF models trained on different subjects. We first segment the SMPL mesh according to the body parts, such that each triangle face belongs to a category (\eg, head, upper body, and lower body). We then decide which subject's appearance to use for each category (\eg, we use the upper body of subject 1 and the lower body of subject 2). For a query point $\textbf{x}$, we project it onto the mesh surface point $\mathbf{s}$ using the proposed dispersed projection and find the category to which the $\mathbf{s}$ belongs. Given this category, we compute the color and the density of point $\textbf{x}$ using the trained NeRF model of the corresponding subject. Finally, we can render the novel synthesized image with a combination of the appearance of multiple subjects. The synthesized results are shown in \cref{fig:clothes_change}.

\section{Discussion}
\label{sec:discuss}

\subsection{Limitations}
Currently, our method relies on a relatively accurate mesh estimation. Although we can alleviate inaccuracies by optimizing the SMPL parameters during training, some details that cannot be accurately represented through SMPL, such as hand pose, may lead to artifacts because points cannot be projected correctly onto the mesh surface. A possible solution is to replace SMPL with parametric human models in more detail, such as with \mbox{SMPL-X}~\cite{SMPL-X:2019}.

The proposed surface-aligned representation using only one surface point may have potential issues that, in some specific poses such as when the arm and the body are extremely close, some points that are projected to the body in the training pose for representing the body information, may be projected to the arm in the novel pose setting, causing that point cannot correctly represent the body information. 
It would be interesting to introduce more information, such as interrelationships of spatial points with each body part as similar to~\cite{2021narf} for further improvement.

Our dispersed projection assumes a watertight mesh with vertex normals all forming acute angles with neighboring face normals, which may be restrictive when applied to more complicated meshes.
Designing a more flexible projection method without the indistinguishability issue would broaden the applicability of surface-aligned representation.

\subsection{Future work}
The core idea of our method is to build a NeRF on the estimated mesh surface. Therefore, our method can be generalized to objects with corresponding 3D mesh, not limited to the human body, with the help of the existing 3D mesh reconstruction methods from images~\cite{kato2018renderer, kato2019vpl, liu2019softras, cmrKanazawa18, chen2019dibrender} or videos~\cite{vmr2020, Novotny_2017_ICCV, Henzler_2021_CVPR, yang2021lasr, reizenstein21co3d}.
Furthermore, combined with interactive 3D mesh deformation methods~\cite{igarashi1999sketch, jin2003sketch, kho2005sketch}, our method can enable interactive manipulations of NeRF. We expect the idea of combining the controllability of mesh and photorealistic rendering of NeRF can be used for more visual applications in the future.

\section{Conclusion}
\label{sec:conclusion}
In this paper, we present novel surface-aligned neural radiance fields for a controllable 3D human synthesis.
The proposed dispersed projection method transforms spatial points into distinguishable surface points and signed heights, which shows high compatibility with the proposed surface-aligned NeRF.
Our method not only shows a better generalization performance in the novel human pose situations, but it also supports explicit controls such as body shape change and clothes change.

\vspace{1em}
\noindent\textbf{Acknowledgements:} We thank Hiroharu Kato for helpful discussions and comments.

{\small
\bibliographystyle{ieee_fullname}
\bibliography{egbib}
}

\appendix

\section{Algorithms}
We present the pseudocodes of our proposed dispersed projection and vertex normal alignment in \cref{alg:1} and \cref{alg:2}, respectively.

\begin{algorithm}[hbt!]
\caption{Dispersed projection}\label{alg:1}
\KwIn{$\mathbf{x} \in \mathbb{R}^3$, body mesh $\mathcal{M}$}
\KwOut{Projected surface point $\mathbf{s}$}
Initialize $\mathcal{S} := []$ \\
Find the nearest surface point $\tilde{\mathbf{s}}$ to $\mathbf{x}$. \\
Find all triangles containing $\tilde{\mathbf{s}}$, denoted as $\mathcal{T}$. \\
\For{$T \in \mathcal{T}$}{
  vertex normal alignment for $T$. \\
  \If{$\mathbf{x}$ inside the parallel triangle $T'$}{
    barycentric interpolated projection $\mathbf{x} \rightarrow \mathbf{s}_T$. \\
    $\mathcal{S}$.append($\mathbf{s}_T$)
  }
}
\Return $\mathbf{s} := \underset{\mathbf{s}_T}{\mathrm{arg\,min}}\, \| \mathbf{s}_T - \mathbf{x} \|, \mathbf{s}_T \in \mathcal{S}$ \\
\end{algorithm}

\begin{algorithm}[hbt!]
\caption{Vertex normal alignment}\label{alg:2}
\KwIn{Triangle $T$}
\For{$i \in \{1, 2, 3\}$}{
  Compute two edge directions $\mathbf{e}_1, \mathbf{e}_2$ from $\mathbf{v}_i$. \\
  Orthogonally project $\mathbf{n}_i$ on plane $T$ at $\mathbf{p}_i$. \\
  Decomposite $\mathbf{p}_i - \mathbf{v}_i = c_1 \mathbf{e}_1 + c_2 \mathbf{e}_2$. \\
  $\tilde{c}_1, \tilde{c}_2 := \mathrm{max}(0, c_1), \mathrm{max}(0, c_2)$. \tcp{only consider that $\mathbf{p}_i$ falls within the inward region.} 
  $\mathbf{n}_i := \mathbf{n}_i - \tilde{c}_1 \mathbf{e}_1 - \tilde{c}_2 \mathbf{e}_2$. \tcp{alignment.}
  Normalize the length of $\mathbf{n}_i$ to 1.}
\end{algorithm}

\section{Implementation details}
\subsection{Network architecture}
\begin{figure}[t]
  \centering
  \includegraphics[width=\linewidth]{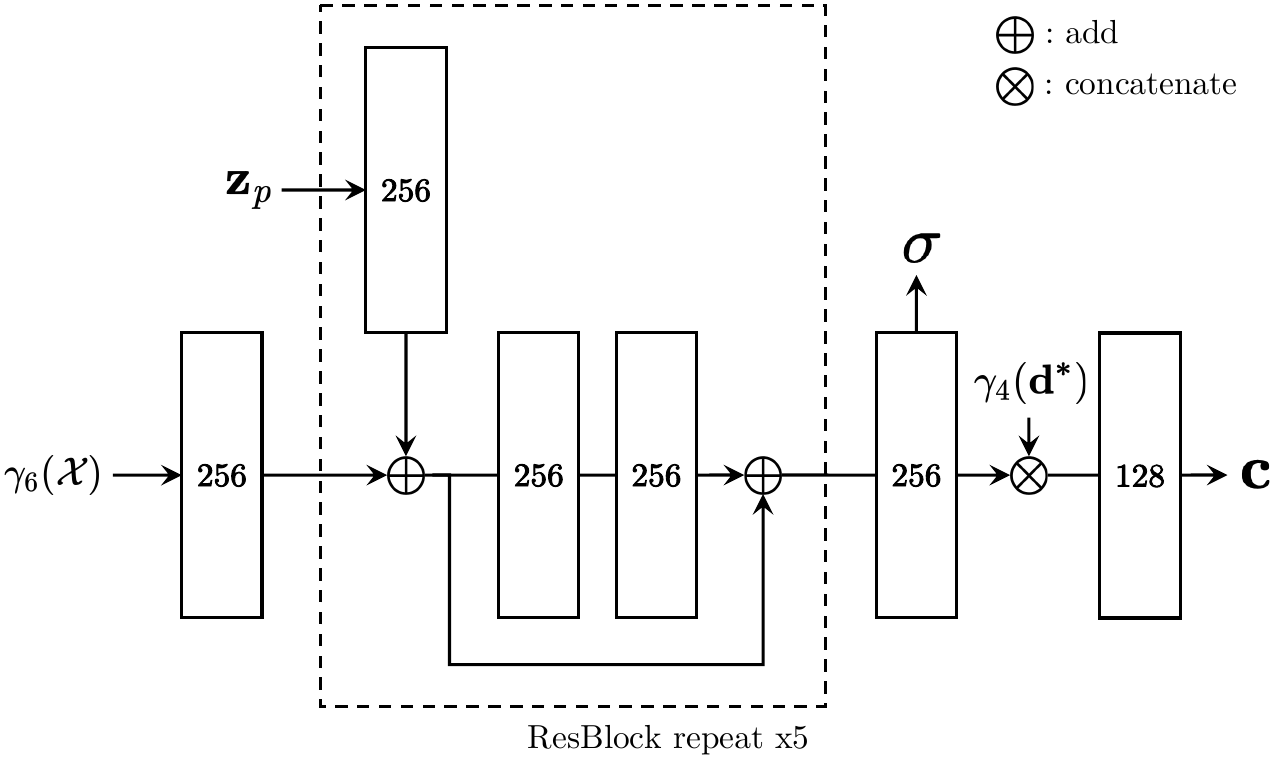}
  \caption{\textbf{Network architecture.} The network takes the positional encoding of the surface-aligned representation $\gamma_6(\mathcal{X})$ and the view direction $\gamma_4(\mathbf{d}^*)$ along with the skeleton pose embedding $\mathbf{z}_p$ and outputs the density $\sigma$ and the RGB color $\mathbf{c}$. The number in each block means the dimension of the input. All linear layers are followed by ReLU activation except the output layers of color and density.}
  \label{fig:arch}
\end{figure}

We present a NeRF network architecture in \cref{fig:arch}. We use positional encoding~\cite{mildenhall2020nerf} with the frequency $L=6$ and $L=4$ for the surface-aligned representation and the view direction, respectively. We use a three-layer, 256-dimensional vanilla graph convolutional network (GCN)~\cite{kipf2017semi} with ReLU activation for encoding the skeleton pose $p \in \mathbb{R}^{24 \times 3}$ to an embedding $\tilde{\mathbf{z}}_p \in \mathbb{R}^{24 \times 256}$ and then perform a spatial average pooling to obtain a latent code $\mathbf{z}_p \in \mathbb{R}^{256}$.

\subsection{Training settings}
We basically refer to \cite{peng2021neural} for the training setting. We use the single-level NeRF and sample 64 points along each camera ray. For points that are far from the predicted SMPL surface, we do not feed them into NeRF for faster training. Specifically, for points with signed distance $h > h_0$, our model returns the zero density and color directly. We set $h_0 = 0.2m$ in our experiments.  We conduct the training on a single Nvidia V100 GPU. Learning rate decreases exponentially from $5e^{-4}$ to $5e^{-5}$ in training. The training typically converges in about 200k iterations, which takes about 14 hours.

\section{Proof of injective mapping}
We prove that the mapping $\mathbf{x} \rightarrow \mathcal{X}$ using proposed dispersed projection is an injection under certain conditions. 
Specifically, we prove that, for spatial point $\mathbf{x} \in \mathbb{R}^3 \backslash \mathbb{I}$, the mapping $\mathbf{x} \rightarrow \mathcal{X}$ is an injection, where $\mathbb{I}$ denotes the set of all bilinear surfaces formed by the adjacent vertex normals after vertex normal alignment for all the triangle faces of a mesh. Since the volume of $\mathbb{I}$ is zero, any spatial point is almost surely in $\mathbb{R}^3 \backslash \mathbb{I}$.

The proof is based on the following assumptions:
\begin{assumption}
\label{ass:1}
Meshes are watertight and do not have triangle faces with zero area.
\end{assumption}
\begin{assumption}
\label{ass:2}
The face normal of and the vertex normals of every triangle form an acute angle. 
\end{assumption}
\begin{assumption}
\label{ass:3}
Any spatial point $\mathbf{x} \in \mathbb{R}^3$ can be projected onto the mesh surface through dispersed projection.
\end{assumption}

From the definition of dispersed projection, we notice that for $\mathbf{x} \in \mathbb{R}^3 \backslash \mathbb{I}$, it will always be mapped to a surface point $\mathbf{s}$ that is strictly inside a triangle face, \ie, not on edges. In this case, the dispersed projection is reduced to a barycentric interpolated projection based on the aligned vertex normals for the triangle face.
We first introduce the following lemma:
\begin{lemma}
\label{lemma:1}
If $\mathbf{x} \in \mathbb{R}^3 \backslash \mathbb{I}$ is outside the mesh and mapped to $\mathbf{s}$ through dispersed projection, then $\mathbf{x} - \mathbf{s} = c\mathbf{n_s}, \ c \in \mathbb{R}^+$ is satisfied for all such $\mathbf{x}$. Here $\mathbf{s}$ is a surface point that is strictly inside a triangle face, and $\mathbf{n_s}$ is a unit vector that is invariant to $\mathbf{x}$.
\end{lemma}

\begin{proof}
Consider a triangle $T$ after vertex normal alignment as in Fig.~3(a) of the main paper. For triangle $T = (\mathbf{v}_1, \mathbf{v}_2, \mathbf{v}_3)$ and its parallel triangle $T' = (\mathbf{v}'_1, \mathbf{v}'_2, \mathbf{v}'_3)$, by the definition of barycentric interpolated projection, we can obtain:

\begin{empheq}[left = {\empheqlbrace \,}, right = {}]{align}
    \mathbf{x} &= \alpha_1 \mathbf{v}'_1 + \alpha_2 \mathbf{v}'_2 + \alpha_3 \mathbf{v}'_3 \label{first equation} \\
    \mathbf{s} &= \alpha_1 \mathbf{v}_1 + \alpha_2 \mathbf{v}_2 + \alpha_3 \mathbf{v}_3 \label{second equation}
\end{empheq}
where $(\alpha_1, \alpha_2, \alpha_3)$ is the barycentric coordinate.
Combining the above equations we can obtain:
\begin{align}
    \mathbf{x}-\mathbf{s} 
    &= \alpha_1 (\mathbf{v}'_1 - \mathbf{v}_1) + \alpha_2 (\mathbf{v}'_2 - \mathbf{v}_2) + \alpha_3 (\mathbf{v}'_3 - \mathbf{v}_3) \label{eq:3} \\
    &= \sum_{i=1,2,3} \alpha_i (\mathbf{v}_i - \mathbf{v}'_i) \label{eq:4}
\end{align}
\begin{equation}
\label{eq:5}
\mathbf{v}_i - \mathbf{v}'_i = \norm{\mathbf{v}_i - \mathbf{v}'_i} \mathbf{n}_{\mathbf{v}_i} 
                                 = \left( l \mathbin{/} \langle \mathbf{n}_{\mathbf{v}_i}, \mathbf{n}_{T} \rangle \right) \mathbf{n}_{\mathbf{v}_i},
\end{equation}
where $l$ denotes the distance $l$ between plane $T$ and $T'$, $\mathbf{n}_{\mathbf{v}_i}$ denotes the aligned vertex normal at $\mathbf{v}_i$, and $\mathbf{n}_T$ denotes the surface normal of $T$.
Substituting \cref{eq:5} for \cref{eq:4} leads to:
\begin{align}
    \mathbf{x}-\mathbf{s} 
    &= \sum_{i=1,2,3} \alpha_i \left( l \mathbin{/} \langle \mathbf{n}_{\mathbf{v}_i}, \mathbf{n}_{T} \rangle \right) \mathbf{n}_{\mathbf{v}_i} \\
    &= l \left( \sum_{i=1,2,3} \left( \alpha_i \mathbin{/} \langle \mathbf{n}_{\mathbf{v}_i}, \mathbf{n}_{T} \rangle \right) \mathbf{n}_{\mathbf{v}_i} \right) \label{eq:7}
\end{align}
The term inside the parentheses of \cref{eq:7} is invariant to $\mathbf{x}$, and we describe its direction by a unit vector $\mathbf{n_s}$, which gives:
\begin{equation}
    \label{eq:8}
    \mathbf{x} - \mathbf{s} = c\mathbf{n_s}, \ c \in \mathbb{R}^+
\end{equation}
This concludes the proof.
\end{proof}

We call $\mathbf{n_s}$ in \cref{eq:8} an \textit{interpolated normal at $\mathbf{s}$}, which only depends on the barycentric coordinates of $\mathbf{s}$ and the mesh with aligned vertex normals. We then introduce the following lemma:
\begin{lemma}
\label{lemma:2}
For $\mathbf{x} \in \mathbb{R}^3 \backslash \mathbb{I}$, the mapping $\mathbf{x} \rightarrow \left( \mathbf{s}, h \right)$ is a injection, where $h = \norm{\mathbf{x} - \mathbf{s}}$ (when $\mathbf{x}$ is outside the mesh) or $h = -\norm{\mathbf{x} - \mathbf{s}}$ (when $\mathbf{x}$ is inside the mesh).
\end{lemma}

\begin{proof}
We first prove the case of when $\mathbf{x}$ is outside the mesh, \ie, $h = \norm{\mathbf{x} - \mathbf{s}}$.
From the definition of $h$, it is exactly $c$ in \cref{eq:8}, which gives:
\begin{equation}
    \label{eq:9}
    \mathbf{x} - \mathbf{s} = h\mathbf{n_s}
\end{equation}
Consider the following:
\begin{empheq}[left = {\empheqlbrace \,}, right = {}]{align}
    \mathbf{x}_1 - \mathbf{s}_1 = h_1\mathbf{n}_{\mathbf{s}_1} \label{eq:10} \\
    \mathbf{x}_2 - \mathbf{s}_2 = h_2\mathbf{n}_{\mathbf{s}_2} \label{eq:11}
\end{empheq}
where $\mathbf{s}_1 = \mathbf{s}_2$, $h_1 = h_2$. For the same surface point $\mathbf{s}_1 = \mathbf{s}_2$, they have the same \textit{interpolated normal}, that is, $\mathbf{n}_{\mathbf{s}_1} = \mathbf{n}_{\mathbf{s}_2}$.
Therefore, from \cref{eq:10} and \cref{eq:11}, we can obtain:
\begin{equation}
    \mathbf{x}_1 = \mathbf{x}_2
\end{equation}
Above equations indicate that:
\begin{equation}
    \label{eq:13}
    \left( \mathbf{s}_1, h_1 \right) = \left( \mathbf{s}_2, h_2 \right) \Rightarrow \mathbf{x}_1 = \mathbf{x}_2
\end{equation}
which concludes that the mapping $\mathbf{x} \rightarrow \left( \mathbf{s}, h \right)$ is a injection.
The case of when $\mathbf{x}$ is inside the mesh is proved similarly by inverting the face and vertex normals of a triangle face and applying \cref{lemma:1}.
\end{proof}

We finally introduce the following lemma:
\begin{lemma}
\label{lemma:3}
The mapping $ \mathbf{s}\rightarrow  \mathbf{s}_c$ is a injection, where $\mathbf{s}_c$ is the corresponding surface point on the T-pose mesh.
\end{lemma}
\begin{proof}
From Assumption~\ref{ass:2}, the shared T-pose mesh is watertight and thus does not have self-intersection. From the definition of $\mathbf{s}_c$, that is, $\mathbf{s}_c$ is inside the same triangle with the same barycentric coordinates as $\mathbf{s}$, the proof is trivial.
\end{proof}

From \cref{lemma:2} and \cref{lemma:3}, we prove that for $\mathbf{x} \in \mathbb{R}^3 \backslash \mathbb{I}$, the mapping $\mathbf{x} \rightarrow \mathcal{X}$ is an injection.
\begin{proof}
\cref{lemma:3} indicates that
\begin{equation}
    \label{eq:14}
    {\mathbf{s}_c}_1 = {\mathbf{s}_c}_2 \Rightarrow \mathbf{s}_1 = \mathbf{s}_2.
\end{equation}
Considering $h_1 = h_2 (h_1, h_2 \in \mathbb{R})$, it immediately follows that 
\begin{equation}
    \label{eq:15}
    ({\mathbf{s}_c}_1, h_1) = ({\mathbf{s}_c}_2, h_2) \Rightarrow (\mathbf{s}_1, h_1) = (\mathbf{s}_2, h_2).
\end{equation}
Combining \cref{eq:13} and \cref{eq:15}, we can obtain
\begin{equation}
    ({\mathbf{s}_c}_1, h_1) = ({\mathbf{s}_c}_2, h_2) \Rightarrow \mathbf{x}_1 = \mathbf{x}_2,
\end{equation}
which is exactly
\begin{equation}
    \mathcal{X}_1 = \mathcal{X}_2 \Rightarrow \mathbf{x}_1 = \mathbf{x}_2.
\end{equation}
This concludes the proof.
\end{proof}

\end{document}